\documentclass[11pt]{article} %

\usepackage[utf8]{inputenc} %

\usepackage{geometry} %
\usepackage{graphicx} %

\usepackage[parfill]{parskip} %

\usepackage{booktabs} %
\usepackage{array} %
\usepackage{paralist} %
\usepackage{verbatim} %

\usepackage{amsmath}
\usepackage{amssymb}
\usepackage{amsthm}
\begingroup
    \makeatletter
    \@for\theoremstyle:=definition,remark,plain\do{%
        \expandafter\g@addto@macro\csname th@\theoremstyle\endcsname{%
            \addtolength\thm@preskip\parskip
            }%
        }
\endgroup

\usepackage{algorithmic}
\usepackage{algorithm}

\usepackage{fullpage}

\usepackage{color}

\newtheorem{proposition}{Proposition}
\newtheorem{assumption}{Assumption}[section]
\newtheorem{lemma}{Lemma}
\newtheorem{corollary}{Corollary}[section]

\newtheorem{theorem}{Theorem}
\theoremstyle{remark}
\newtheorem{example}{Example}[section]

\usepackage{dsfont}
\usepackage{amsthm}
\usepackage{graphicx}
\usepackage{caption}
\usepackage{subcaption}
\usepackage[round]{natbib}

\newcommand{\rhop}{\rho_+}
\newcommand{\rhom}{\rho_-}

\newcommand{\calf}{\mathcal{F}}
\newcommand{\calh}{\mathcal{H}}

\newcommand{\calr}{\mathcal{R}}
\newcommand{\calx}{\mathcal{X}}
\newcommand{\caly}{\mathcal{Y}}
\newcommand{\cala}{\mathcal{A}}
\newcommand{\calz}{\mathcal{Z}}
\newcommand{\calp}{\mathcal{P}}

\newcommand{\bx}{\mathbf{x}}
\newcommand{\bX}{\mathbf{X}}

\newcommand{\bbe}{\mathbb{E}}
\newcommand{\bbp}{\mathbb{P}}
\newcommand{\bbr}{\mathbb{R}}

\newcommand{\sgn}{\text{sgn}}

\newcommand{\pn}{P_n}
\newcommand{\tpn}{\tilde{P}_n}

\newcommand{\tcalr}{\tilde{\mathcal{R}}}

\newcommand{\nes}{\tilde{f}_n}

\usepackage{fancyhdr} %
\title{Binary Classification with Instance and Label Dependent Label Noise}
\author{Hyungki Im$^1$, Paul Grigas$^1$ \\
  $^1$Department of Industrial Engineering and Operations Research\\
  University of California, Berkeley\footnote{\{hyungki.im, pgrigas\}@berkeley.edu}
}

\begin{document}
\maketitle

\begin{abstract}
 Learning with label dependent label noise has been extensively explored in both theory and practice; however, dealing with instance (i.e., feature) and label dependent label noise continues to be a challenging task. The difficulty arises from the fact that the noise rate varies for each instance, making it challenging to estimate accurately. The question of whether it is possible to learn a reliable model using only noisy samples remains unresolved. We answer this question with a theoretical analysis that provides matching upper and lower bounds. Surprisingly, our results show that, without any additional assumptions, empirical risk minimization achieves the optimal excess risk bound. Specifically, we derive a novel excess risk bound proportional to the noise level, which holds in very general settings, by comparing the empirical risk minimizers obtained from clean samples and noisy samples. Second, we show that the minimax lower bound for the 0-1 loss is a constant proportional to the average noise rate. Our findings suggest that learning solely with noisy samples is impossible without access to clean samples or strong assumptions on the distribution of the data.
\end{abstract}

\section{Introduction}\label{sec:intro}

Within the framework of the traditional classification problem, it is generally assumed that both the training data and the test data are drawn from identical distributions.
Nevertheless, there are several factors, including covariate shift, disturbances to the data, and changes in the domain of the application that can cause the distribution of testing data to be very different from that of training data. A particularly important example is \emph{label noise}, where the labels in a classification problem may be corrupted with some nonzero probability.
Various factors can contribute to the emergence of label noise, including human error during the annotation process, instances of ambiguous classification boundaries, and inconsistencies in data collection and processing. Also, with the increasing scale of data in recent years, researchers and practitioners are increasingly turning to cost-effective data collection methods, such as data mining on social media \citep{injadat2016data} or web scraping \citep{schroff2010harvesting}, rather than relying on human experts. 
These cost-effective approaches are more susceptible to the presence of label noise, which has contributed to the increased prevalence of this issue in recent years. 
Consequently, a critical question within the machine learning community is to understand the effects of label noise and when one may or may not be able to develop algorithms that are robust to the noise.
\par 
In this paper, we consider a binary classification problem with a feature space $\calx \subset \bbr^d$ and a label space $\caly = \{-1,+1\}$. 
We adopt the label noise model using the label noise function $\rho_y(\bx)$, which quantifies the probability of flipping the label $y$ to -$y$ given the sample $(\bx,y)$. 
The simplest possible label noise model is the \textit{random classification noise} (RCN) model \citep{angluin1988learning}, in which the label noise function remains constant $\rho_y(\bx) = \rho$ for all $x \in \calx$ and $y \in \caly$. 
The RCN model has been extensively studied in theory, and it is well established that certain loss functions, such as the 0-1 loss or squared loss, can learn an optimal classifier solely using noisy samples even when the noise rate $\rho$ is unknown \citep{manwani2013noise, ghosh2015making}.
\par 
The \textit{class-conditional noise} (CCN) model is more complex than the RCN model, wherein the label noise function is a constant $\rho_y$ depending on each label $y \in \caly$ \citep{natarajan2013learning}. 
Similarly to the RCN model, numerous theoretical works and methods have demonstrated that learning with the CCN is possible
\citep{menon2015learning,natarajan2013learning,scott2015rate,van2017theory,reeve2019classification,zhang2021learning}. A common approach to tackle CCN is the \textit{loss correction} method, which adjusts the loss function according to the estimated noise rate, ensuring the consistency of the estimator for clean samples.
Some loss correction methods necessitate knowing the noise rates \citep{menon2015learning, van2017theory, patrini2017making}, while others estimate these rates under the assumption that anchor points exist \citep{scott2013classification,scott2015rate,ramaswamy2016mixture,liu2015classification, reeve2019classification}. An instance $\bx$ is an anchor point for the label $y \in \caly$ if $\bbp(Y=y|\bX=\bx)=1$ and it is useful in the sense that it allows us to learn the noise rates accurately. Moreover, \citet{xia2019anchor} proposed a noise rate estimator that does not require the anchor points. Recently, \citet{liu2020peer} introduced a peer loss function that allows learning with noisy samples without the need to estimate noise rates. Other approaches to handle CCN include \textit{sample selection} method and \textit{label correction} method, in which they aim to identify noisy samples, filter them out, or correct them to the accurate label \citep{bootkrajang2012label,goldberger2017training,malach2017decoupling,han2018co,jiang2018mentornet,huang2019o2u,nguyen2019self,yu2019does,wei2020combating,yi2019probabilistic,han2020sigua,park2021wasserstein}. 
However, it is intuitive that the label noise in real-world data should be dependent on both the instance and the label rather than solely on the label. For example, a blurred image is more likely to be mislabeled compared to a clear version of the same image. Additionally, \citet{chen2021beyond} provided theoretical evidence that the CCN assumption does not hold in real-world datasets by introducing a CCN hypothesis testing method. Thus, this motivates us to examine a more general label noise model, \textit{instance and label dependent noise} (ILN) model, in which the noise rate depends on both the instance and the label.
\par 
The purely instance-dependent noise (PIN) model, where the noise rate depends only on the instance and not the label, has been studied by \citet{menon2018learning} as a special case of the ILN model. 
\citet{menon2018learning} showed that the Bayes optimal classifiers for both the clean and noisy distributions are identical, indicating that learning with PIN is possible.
However, in the ILN model, the problem becomes more challenging, as estimating the noise rate for each instance is generally difficult. 
There are some algorithms proposed to address learning under the ILN model \citep{cheng2020learning, xia2020part, cheng2020learning-sample, zhu2021second}.
\citet{xia2020part} attempted to address the ILN model by making a part dependence assumption, which assumes that instances are composed of parts and that these parts are subject to corruption. 
\citet{cheng2020learning} considered the \textit{bounded instance and label dependent} (BILN) model, where the noise rate for each label is bounded above. 
They tried to learn an estimator using samples that are consistent with the Bayes optimal classifier and samples that have been cleaned by human experts. 
\citet{berthon2021confidence} considered the \textit{confidence-scored instance-dependent} model, in which the confidence score information is provided along with the instance and the label. 
Although these methods provide us with heuristic ways to address the ILN, there has been limited theoretical work on the ILN model. For example, each of these methods requires access to clean samples or relies on specific strong assumptions. Nevertheless, the question of whether learning with ILN is possible without these conditions or access to clean samples has not been adequately addressed.
Furthermore, while much work has been done to modify the loss function to enable empirical risk minimization (ERM) to learn from noisy samples, it is often overlooked how the label noise might degrade the performance of ERM.
 Perhaps the work closest to these lines of questions and to this paper is the work of \citet{lee2022binary}. They studied the performance of ERM in the presence of RCN. However, their work is restricted to the RCN model with Lipschitz loss functions and linear hypothesis classes, whereas this study explores more comprehensive settings: \textit{instance and label dependent noise}, as well as a broader range of hypothesis classes and loss functions.

\par 
In this paper, we provide a novel theoretical analysis of the instance and label dependent (ILN) model. 
We consider loss functions that satisfy a mild boundedness assumption and ensure that standard generalization guarantees hold. These conditions do not require the loss function to be Lipschitz or convex, and even 0-1 loss complies with these requirements.
Moreover, other general loss functions such as logistic loss or hinge loss, meet these conditions under certain settings.
We start by studying the performance difference between the clean estimator and the noisy estimator, both obtained through empirical risk minimization (ERM) using clean samples and noisy samples, respectively. (Herein, we refer to a clean sample as a sample from the underlying distribution, and a noisy sample as one that has passed through the ILN model). 
We evaluate the performance of each estimator with respect to the risk under the clean (true) distribution and provide an upper bound for the difference between those risks. We identify a new source of \emph{irreducible error}, namely a constant term proportional to the amount of label noise present, in the upper bound. 
Next,
we show that this irreducible error term cannot be fully removed in the ILN model by developing a minimax lower bound of the 0-1 risk. This suggests that to learn an effective predictor under the ILN model, we may need access to clean samples \citep{cheng2020learning} or require some additional assumptions \citep{xia2020part} beyond those made in the CCN model.

\paragraph{Organization} The rest of the paper is organized as follows. In Section \ref{sec:problem setting}, we introduce our problem setting for the ILN model. In Section \ref{sec:risk gap}, we derive an upper bound for the difference between the risk of the clean estimator and the risk of the noisy estimator.  
In Section \ref{sec:lower bound}, we show that the minimax lower bound of the 0-1 risk is a constant proportional to the noise level, even when considering the margin assumption \citep{mammen1999smooth} and the anchor point assumption, which are commonly used in the CCN model.

\section{Problem setting}\label{sec:problem setting}

In this section, we present our formal problem setting and assumptions of the ILN model. Let $\calx \subset \bbr^d$ be the feature space and $\caly =\{-1,+1\}$ be the label space. We assume that the random variables $(\bX, Y, \tilde{Y}) \in \calx \times \caly \times \caly$ are jointly distributed with unknown distribution $D$, where $\bX$ denotes the feature vector, $Y$ is the true label, and $\tilde{Y}$ is the noisy label. We use $P_\bX$ to denote the marginal distribution of $D$ for $\bX$, while $P$ and $\tilde{P}$ are used to indicate the marginal distribution of $D$ for $(\bX,Y)$ and the marginal distribution of $D$ for $(\bX,\tilde{Y})$, respectively. We refer to $P$ as a clean distribution as it is related to the true label $Y$ and refer to $\tilde{P}$ as a noisy distribution.
We use $S_n$ and $\tilde{S}_n$ to denote $n$ i.i.d. samples drawn from $P$ and $\tilde{P}$, respectively. Furthermore, we use $\pn$ and $\tpn$ to represent the corresponding product distribution associated with the samples $S_n$ and $\tilde{S}_n$.
Given the noise functions $\rho_y(\cdot): \calx \rightarrow [0,1)$ for $y \in \{-1,+1\}$, a sample of $(\bX, Y, \tilde{Y})$ is generated by the following procedure:
\begin{align*}
    \bX \sim P_\bX, \ Y|X &= \begin{cases} +1, \quad \text{with prob. } \eta(\bX) \\ -1, \quad \text{with prob. } 1-\eta(\bX) \end{cases} \\
    \tilde{Y}|Y,X &= \begin{cases}
        -Y, \quad \text{with prob. } \rho_Y(\bX) \\
        Y,  \quad \ \ \ \text{with prob. } 1-\rho_Y(\bX).
    \end{cases}
\end{align*}
Here $\eta(\bx):= \bbp(Y=1|\bX=\bx)$, denotes the conditional probability of $Y=1$ given $\bX = \bx$ (i.e., the regression function), while $\rho_y(\bx)$ indicates the probabilty of the true label $y$ being flipped for instance $(\bx,y).$ Both $\eta(\cdot)$ and $\rho_y(\cdot)$ for $y \in \{\-1,+1\}$ are unknown. For simplicity, we use $\rhop(\bx)$ and $\rhom(\bx)$ to denote $\rho_{+1}(\bx)$ and $\rho_{-1}(\bx)$, respectively.
As discussed in Section \ref{sec:intro}, various noise models have been considered in the literature and each of these models can be formulated using the noise functions $\rho_y(\bx)$.
Precisely, the RCN model is a special case where the noise function is given by $\rhop(\bx) = \rhom(\bx) = \rho \in [0,1)$ for all $\bx \in \calx$. The CCN model generalizes the RCN model, with constant noise levels $\rhop \in [0,1)$ and $\rhom \in [0,1)$ such that $\rhop(\bx) = \rhop$ and $\rhom(\bx) = \rhom$ for all $\bx \in \calx$.
For the ILN model, there are no specific constraints on the noise functions, covering the more general scenarios. In this study, we make the following mild assumption for the ILN model.
\begin{assumption}\label{assm:noisy}
For all $x \in \calx$, we have
\begin{align*}
    &0 \leq \rhop(x) + \rhom(x) \leq \rho<1.
\end{align*}
\end{assumption}

\par 
In this case, the constant $\rho$ represents a uniform bound on the sum of the two noise functions for all $x \in \calx$.
The objective in a binary classification problem with label noise is to determine a function $f: \calx \rightarrow \bbr$ such that $\sgn(f(\bx))$ best predicts the true label $y$, for a given input $(\bx,y)$, where $\sgn(\cdot)$ represents the sign function. However, we only observe a sequence of noisy pairs $\tilde{S}_n=\{\bx_i, \tilde{y}_i\}_{i=1}^n$, and the values of the true labels $\{y_i\}_{i=1}^n$ are unknown.
As is standard in binary classification, a fundamental performance metric of the function $f(\cdot)$ is the 0-1 risk
\begin{align*}
    \calr^{0/1}(f) = \bbe_{(\bX,Y) \sim P} \left[ \mathds{1} \{ \sgn(f(\bX)) \neq Y \}  \right],
\end{align*}
where $\mathds{1}\{\cdot\}$ denotes the standard 0-1 indicator function.
The function $f^*(\bx):= \eta(\bx) -\frac{1}{2}$ associated with the Bayes optimal classifier minimizes the 0-1 risk, and $\calr^{0/1}(f^*)$ indicates the minimum misclassification rate under distribution $P$. Since the regression function $\eta(\cdot)$ is unknown, the learner must identify a function $f: \calx \rightarrow \bbr$ that approximates $\eta(\cdot)$ based on the observed samples. 
Additionally, the distribution $P$ remains unknown and therefore the 0-1 risk can not be evaluated. As an alternative, the learner generally seeks a decision rule that (approximately) minimizes the empirical 0-1 risk 
\begin{align*}
    \calr^{0/1}_n(f):= \frac{1}{n} \sum_{i=1}^n \mathds{1}\{ \sgn(f(\bx_i)) \neq y_i \}.  
\end{align*}
In the context of the noisy label setting where the true labels $\{y_i\}_{i=1}^n$ are unobserved, the learner may in fact instead (approximately) minimize the empirical 0-1 risk with noisy labels 
\begin{align*}
    \tcalr^{0/1}_n(f):= \frac{1}{n} \sum_{i=1}^n \mathds{1}\{ \sgn(f(\bx_i)) \neq \tilde{y}_i \}.
\end{align*}
The non-convex and discontinuous nature of the 0-1 loss makes the above minimization challenging, and the classical way to circumvent this issue is to consider an alternative surrogate convex loss function $\ell(f(\bx),y)$ \citep{bartlett2006convexity}.
The risk associated with the surrogate loss function $\ell(f(\bx),y)$ is referred as $\ell$-risk of function $f$, and is denoted by
\begin{align*}
    \calr(f):= \bbe_{(\bX,Y) \sim P} \left[ \ell(f(\bX),Y)  \right], \quad \tcalr(f):= \bbe_{(\bX,\tilde{Y}) \sim \tilde{P}} \left[ \ell(f(\bX),\tilde{Y}))  \right],
\end{align*}
where $\calr(f)$ denotes the $\ell$-risk under the clean distribution and $\tcalr(f)$ denotes the $\ell$-risk under the noisy distribution.
According to the above notation, we define $\calr_n(f)$ and $\tcalr_n(f)$ as the empirical $\ell$-risk of a function $f$ with clean samples and noisy samples, respectively. Generally, we select the function $f$ from a predefined class of functions $\calf$, which represents the hypothesis class in our modeling framework.
We denote the empirical risk minimizer with clean samples as:
\begin{align}
    f_n \in \underset{f \in \calf}{\text{argmin}} \ \calr_n(f), \label{eqn:clean estimator}
\end{align}
and for noisy samples, we have:
\begin{align}
    \nes \in \underset{f \in \calf}{\text{argmin}} \ \tcalr_n(f).
    \label{eqn:noisy estimator}
\end{align}
For clarity, we call $f_n$ a clean estimator and $\nes$ a noisy estimator. Note that the loss function remains unchanged when obtaining the noisy estimator $\nes$, which distinguishes it from the \textit{loss correction} or the \textit{label correction} methods discussed in Section \ref{sec:intro}.
Although empirical risk minimization (ERM) is a fundamental technique in machine learning and many works focused on modifying the loss function to learn with ERM using noisy samples, the question of how ERM degrades in the presence of label noise has been previously overlooked.
Therefore, one of our interests is to study the performance difference between the noisy estimator and the clean estimator. 

\section{Risk gap between noisy estimator and clean estimator}\label{sec:risk gap}

In this section, we study the performance difference between the clean estimator $f_n$ and the noisy estimator $\nes$, each of which are obtained through empirical risk minimization (ERM) using clean samples and noisy samples, respectively. In particular, the performance difference can be expressed as:
\begin{align}
    \calr(\nes) - \calr(f_n). \label{eqn:risk gap}
\end{align}
It is important to emphasize that we do not apply any modification to the loss function $\ell(\cdot)$ to address the issue of noisy samples. We consider loss functions that adhere to the following mild assumption. 
\begin{assumption}\label{assm:constant C}
    There is a constant $C$ such that 
    \begin{align}
        |\ell(f(\bx),1) - \ell(f(\bx),-1)| \leq C, \quad \text{for all }\bx \in\calx \text{ and } f \in\calf. \label{eqn:constant C}
    \end{align}
\end{assumption}
Assumption \ref{assm:constant C} states that the difference between the loss of input $(\bx,1)$ and input $(\bx,-1)$ is uniformly bounded over $\calx$ and $\calf$. This assumption is typical in the problem of binary classification. For example, the 0-1 loss satisfies \eqref{eqn:constant C} with $C=1$ for any $\calf$ and $\calx$. 
In addition, other loss functions, such as squared loss and logistic loss, satisfy \eqref{eqn:constant C} for specific $\calx$ and $\calf$. We introduce some relevant examples in section \ref{subsec:examples}. Furthermore, we make a natural assumption about the generalization bound of the empirical risk minimization (ERM) procedure. 
\begin{assumption}\label{assm:ERM}
For any $\delta \in (0,1]$, there exists a function $G_\delta: \mathbb{N} \rightarrow \bbr$ with $G_\delta(n) \rightarrow 0 \text{ as $n\rightarrow \infty$ }$, and for all distributions $P$, it satisfies
    \begin{align*}
         \underset{f \in \calf}{\sup} \ \left| \calr(f) - \calr_n(f) \right| \leq G_\delta(n),
    \end{align*}
with probability at least $1-\delta$.
\end{assumption}
Note that the function $G_\delta$ in Assumption \ref{assm:ERM} may depend on $\mathcal{F}$, $\mathcal{X}$, and the loss function $\ell$, but not the distribution $P$. In particular, Assumption \ref{assm:ERM} also implies
\begin{align*}
    \underset{f \in \calf}{\sup} \ \left| \tcalr(f) - \tcalr_n(f) \right| \leq G_\delta(n),
\end{align*}
for the noisy distribution $\tilde{P}$ as well.
Assumption \ref{assm:ERM} is valid in various settings, such as a linear hypothesis class or a Reproducing Kernel Hilbert Space (RKHS) with a bounded kernel. Also, the function $G_\delta$ is closely connected to the complexity of the hypothesis class $\calf$, which can be quantified through measures such as the VC dimension, Rademacher complexity, and $\epsilon$-covering number \citep{wainwright2019high}. The details are discussed in Section \ref{subsec:examples}.
\subsection{Upper bound analysis}
In this section, we derive an upper bound for the risk gap between the clean and the noisy estimator, as defined in Equation \eqref{eqn:risk gap}.
We start by studying the relationship between the clean conditional distribution $\eta(\bx)$ and the noisy conditional distribution $\tilde{\eta}(\bx):= \bbp(\tilde{y}=1|\bX= \bx)$.
\begin{lemma}\label{lem:eta}
Considering the ILN model, we have
\begin{align}
         \tilde{\eta}(\bx) = (1-\rhop(\bx))\eta(\bx) + \rhom(\bx)(1- \eta(\bx)) \quad \text{for all } x \in \calx. \label{eqn:rhop}
\end{align}
\end{lemma}
The equation \eqref{eqn:rhop} makes intuitive sense within the framework of our noise model, as the noisy label $\tilde{y}=1$ can come from either an uncorrupted true label $y=1$ or a corrupted true label $y=-1$. To obtain an upper bound of the risk gap, we need to study the difference between the clean risk and the noisy risk.
Using the relationship between $\eta(\bx)$ and $\tilde{\eta}(\bx)$, we derive an upper bound on the difference between the clean risk and the noisy risk uniformly over $f \in\calf$.

\begin{lemma}\label{lem:ub of clean-noisy}
Considering the ILN model under Assumptions \ref{assm:constant C} and \ref{assm:ERM}, we have
\begin{align}
    \underset{f \in \calf}{\sup} \big\{ \calr(f) - \tcalr(f) \big\} \leq \frac{3C\rho}{2}.
\end{align}
\end{lemma}

The upper bound in Lemma \ref{lem:ub of clean-noisy} is intuitive. When there is no noise in the distribution, $\rho=0$, the upper bound is equal to 0. Additionally, as the amount of noise present in the distribution increases, the upper bound also increases proportionally to the noise rate. Based on these findings, we introduce our main theorem, which provides an upper bound on the risk gap between the clean estimator and the noisy estimator.
\begin{theorem}\label{thm:upper}
Consider the ILN model under Assumptions \ref{assm:constant C} and \ref{assm:ERM}. Let $(\bx_1,y_1,\tilde{y}_1)$,$\cdots,$ $(\bx_n,y_n, \tilde{y}_n)$ be an i.i.d. sample drawn from the distribution $D$, and $f_n$ and $\nes$ be the clean estimator and the noisy estimator as defined in Equation \eqref{eqn:clean estimator} and \eqref{eqn:noisy estimator}.
Then, for any $\delta \in (0,1]$, with probability at least $1-2\delta$, we have
\begin{align}
    \calr(\nes) - \calr(f_n) \leq 3C\rho + 2G_\delta(n).\label{eqn:upper thm}
\end{align}
\end{theorem}

We can see that a constant term related to the constant $\rho$ appears in the upper bound in Theorem \ref{thm:upper}. Similar to the result of Lemma \ref{lem:ub of clean-noisy}, the structure of \eqref{eqn:upper thm} appears to be reasonable. If the noise level $\rho$ is small and the sample size $n$ is large enough, the performance difference between the clean estimator and the noisy estimator becomes negligible. On the other hand, if there is a significant amount of noise in the data, the performance difference between the clean estimator and the noisy estimator grows substantially. Using the fundamental relationship between the excess risk and generalization bound and Theorem \ref{thm:upper}, we can derive an upper bound for the excess risk of the noisy estimator $\nes$.
\begin{corollary}\label{cor:excess}
Consider the ILN model under Assumptions \ref{assm:constant C} and \ref{assm:ERM}. Let $(\bx_1,\tilde{y}_1)$,$\cdots,$ $(\bx_n, \tilde{y}_n)$ be an i.i.d. sample drawn from the distribution $\tilde{P}$, and $\nes$ be the noisy estimator as defined in Equation \eqref{eqn:noisy estimator}.
Then, for any $\delta \in (0,1]$, with probability at least $1-4\delta$, we have
\begin{align}
\calr(\nes) - \underset{f \in\calf}{\inf} \calr(f) \leq 3C\rho +  4G_\delta(n). \label{eqn:upper}
\end{align}
\end{corollary}
\begin{proof}
    With probability $1-2\delta$, we have
    \begin{align*}
        \calr(f_n) - \underset{f \in \calf}{\inf} \ R(f) \leq  2G_\delta(n).
    \end{align*}
    Then, adding the above inequality to inequality \eqref{eqn:upper thm}, we get the desired result.
\end{proof}

We refer to the constant term $3C\rho$ in the upper bound given by Corollary \ref{cor:excess} as an \textit{irreducible error} which is conceptually similar to the \textit{approximation error}, as it arises from approximating the clean distribution using noisy samples. Furthermore, the term $4G_\delta(n)$ can be interpreted as an \textit{estimation error}.
However, for some special cases, the irreducible error can be zero, which occurs when the regression function $\eta(\bx)$ is learnable only using noisy samples. As we mentioned in Section \ref{sec:intro}, \citet{menon2018learning} showed that it is possible to learn $\eta(\bx)$ through minimizing the empirical risk of noisy samples when the noisy label is purely instance-dependent, where the noise function is identical for all labels and depends solely on the instance $\bx \ (\rhop(\cdot)\equiv\rhom(\cdot))$.
Also, we discussed various approaches that enable learning with CCN in Section \ref{sec:intro}; however, their theoretical understanding remains limited in the context of ILN. 
Therefore, a pertinent question arises: Is there an approach, apart from ERM, that can remove the irreducible error term $3C\rho$ in Corollary \ref{cor:excess} while relying exclusively on noisy samples?
We address this question by studying the minimax risk lower bound in Section \ref{sec:lower bound}.

\subsection{Applications}\label{subsec:examples}
In this section, we present examples of two specific settings where Theorem \ref{thm:upper} can be applied. First, observe that Assumption \ref{assm:constant C} and Assumption \ref{assm:ERM} are quite broad, and this generality enables us to apply our theoretical result to a wide range of settings. Let us consider a bounded domain $\calx:= \{x \in \bbr^d | \|x \|_2\leq X_* \}$. Then, Assumption \ref{assm:constant C} is valid for any Lipschitz loss function $\ell(f(\bx),y)$ and any hypothesis class $\calf$ that is a subset of a class of $L_\calf$-Lipschitz functions, where $L_\calf \geq 0$ is a given constant.
We provide an example for the case when we consider a linear hypothesis class.
\begin{example}\label{ex:linear}
Suppose that we have a bounded domain  $\calx:= \{x \in \bbr^d ~|~ \|x \|_2\leq X_* \}$, a constrained linear hypothesis class $\calf:= \{ x \rightarrow w^\top x \ ~|~ \ \|w\|_2\leq W_* \}$ and a $L$-Lipschitz margin-based loss function,
i.e., $\ell(f(\bx),y) = \ell(y\cdot f(\bx))$. Notice that the linear hypothesis class is a subset of a class of $2X_*W_*$-Lipschitz functions, with logistic loss and hinge loss being examples of such loss functions.
Then, we have
\begin{align*}
    |\ell(f(\bx)) - \ell(-f(\bx))| \leq 2L|f(\bx)| = 2L|w^\top \bx| \leq 2LX_*W_*.
\end{align*}
So Assumption \ref{assm:constant C} holds with constant $C=2LX_*W_*$.
\end{example}
Example \ref{ex:linear} represents a well-studied, classic setting found in machine learning literature. \citet{kakade2008complexity} studied the generalization bounds for the constrained linear hypothesis class, and we review a pertinent result in the following lemma.
\begin{lemma}[Corollary 4 in \citep{kakade2008complexity}]\label{lem:linear}
    Suppose we adopt the same setting as in Example \ref{ex:linear}. For any $\delta \in (0,1]$, with probability at least $1-\delta$, we have
    \begin{align*}
        \underset{f \in \calf}{\sup} \left| \calr(f) - \calr_n(f) \right| \leq 2LX_*W_*\sqrt{\frac{1}{n}} + LX_*W_*\sqrt{\frac{\log(1/\delta)}{2n}}.
    \end{align*}
\end{lemma}
Lemma \ref{lem:linear} implies that $G_\delta(n) = 2LX_*W_*\left(\sqrt{\frac{1}{n}} + \sqrt{\frac{\log(1/\delta)}{2n}}\right)$ fulfills Assumption \ref{assm:ERM} in the context of Example \ref{ex:linear}. 
Thus, by applying the above result to Theorem \ref{thm:upper}, we obtain the following upper bound.
\begin{proposition}
Suppose we follow the same setting of Example \ref{ex:linear}. For any $\delta \in (0,1]$, with probability at least $1-4\delta$, we have
\begin{align}
    \calr(\nes) - \calr(f_n) \leq 2LX_*W_*\left(3\rho + 2\sqrt{\frac{1}{n}} + \sqrt{\frac{2\log(1/\delta)}{n}} \right).\label{eqn:linear}
\end{align}
\end{proposition}
\citet{lee2022binary} showed the upper bound of the excess risk of noisy estimator when $\calf$ is a linear hypothesis class without bound constraints. 
In fact, the upper bound in \eqref{eqn:linear} can be derived by following the same proof technique as \citet{lee2022binary}.
However, it is important to note that our theoretical results are considerably more general, as their results were based on the simpler RCN model, while our results were built upon the ILN model. Also, their findings were limited to the linear settings, while our results can be applied to a wide range of classes.
As an example beyond the linear hypothesis class, we consider a case of non-parametric regression.
Let us consider the following example.
\begin{example}\label{ex:kernel} Suppose we have a bounded kernel $k$ with $k(x,x) \leq R^2$ for all $x \in \calx$ and let $\calh$ be its reproducing kernel hilbert space (RKHS), and $L$-Lipschitz margin-based loss function. We consider a hypothesis class $\calf = \{f \in \calh | \|f\|_\calh \leq M \}$, where $\|\cdot\|_\calh$ is the corresponding RKHS norm.
For given $f \in \calf$ and $\bx \in\calx$, we have
\begin{align*}
    \left|\ell(f(\bx)) - \ell (-f(\bx)) \right| \leq 2|f(\bx)| \leq 2L\|f\|_\infty \leq 2LR \|f\|_\calh \leq 2LRM.
\end{align*}
So, Assumption \ref{assm:constant C} holds with constant $C=2LRM$.
\end{example}
Also, we have the following well-known result generalization bound for RKHS \cite[Chapter~6.3]{mohri2018foundations}.
\begin{lemma}\label{lem:kernel}
Suppose we adhere to the same setting as in Example \ref{ex:kernel}. For given $\delta \in (0,1]$, with probability at least $1-\delta$, we have
    \begin{align*}
        \underset{f \in \calf}{\sup} \ \left| \calr(f) - \calr_n(f) \right| \leq 2LRM\left( \sqrt{\frac{1}{n}} + \sqrt{\frac{\log(2/\delta)}{2n}} \right).
    \end{align*}
\end{lemma}
Lemma \ref{lem:kernel} shows us that we can use $G_\delta(n) = 2LRM\left( \sqrt{\frac{1}{n}} + \sqrt{\frac{\log(2/\delta)}{2n}} \right)$. Thus, by applying the above result to Theorem \ref{thm:upper}, we get the following result.
\begin{proposition}
Suppose we follow the setting same as Example \ref{ex:kernel}. For any $\delta \in (0,1]$, with probability at least $1-4\delta$, we have 
\begin{align*}
    \calr(\nes) - \calr(f_n) \leq 2LRM\left(3\rho + 2\sqrt{\frac{1}{n}}  + \sqrt{\frac{\log(2/\delta)}{2n}} \right).
\end{align*}
\end{proposition}
This example shows that our theoretical findings can be applied to a wide range of settings encountered in classification problems.

\section{Lower bound analysis}\label{sec:lower bound}

In this section, we study the minimax lower bound for 0-1 loss for binary classification within the ILN model, to see whether the irreducible error in Equation \eqref{eqn:upper} can be removed.
We would like to show that learning with ILN is not possible under assumptions typically made in the CCN model.
These assumptions strengthen our lower bound results, as they make the clean distribution $P$ more favorable for learning.
We start by introducing two widely adopted assumptions in the context of binary classification with CCN. These two assumptions only apply to the clean distribution $P$. 
\begin{assumption}[Margin assumption \citep{mammen1999smooth} ]\label{assm:margin}
    For given $\alpha \in [0,\infty)$ and $C_\alpha \in [1,\infty)$, the clean distribution $P$ satisfies the margin assumption with parameter $(\alpha,C_\alpha)$ if the following holds for all $\xi \in (0,1)$,
    \begin{align*}
        \bbp\left( \left\{ \bx \in \calx: 0<\left|\eta(\bx)- \frac{1}{2} \right| < \xi \right\}\right) \leq C_\alpha \xi^\alpha.
    \end{align*}
\end{assumption}
The margin assumption is widely used in the problem of binary classification \citep{audibert2007fast,mammen1999smooth,pillaud2018exponential,reeve2019classification}. In binary classification, the easiest instance $\bx$ is where the label is deterministic $(\eta(\bx)=0 \text{ or }1)$, while the hardest instance $\bx$ is where the conditional probability is on the decision boundary ($\eta(\bx)=\frac{1}{2}$). Thus, the difficulty of the problem is determined by the quantity of $\bx \in \calx$ where $\eta(\bx)$ is close to $\frac{1}{2}$. The margin assumption implies that we can control this quantity with the parameter $(\alpha, C_\alpha)$ and also indicates that not all of the mass of $\bx \in \calx$ is concentrated around $\eta(\bx) = \frac{1}{2}$.
\begin{assumption}[Anchor point assumption]\label{assm:anchor}
    For each label $y \in \{-1,+1\}$, there exist $\bx \in \calx$ such that $\bbp(Y = y|\bX = \bx)=1$ and $\bbp(\bX=\bx) > 0$.
\end{assumption}
The anchor point assumption implies that for each label there exists at least one instance where the label is deterministic. The anchor point allows us to estimate the unknown noise rate in the classification problem with CCN \citep{scott2015rate,reeve2019classification}. With two assumptions together, many theoretical works showed that they were able to learn the optimal classifier only with noisy samples when the noise is instance-independent but label-dependent. However, the following theorem implies that this is not possible with instance- and label-dependent label noise. 
\par 

\begin{theorem}\label{thm:lower}
Suppose $\hat{f}_n$ is a function learned using $n$ noisy samples $\tilde{S}_n$. For all $n \in \mathbb{N}$, we have
\begin{align}
    \underset{\hat{f}_n}{\inf} \  \underset{P, \tilde{P}}{\sup} \ \left\{ \mathbb{E}_{\tilde{S}_n \sim \tilde{P}_n}\left[ \calr^{0/1}(\hat{f}_n) - \calr^{0/1}(f^*) \right] \right\} \geq \frac{\rho}{16}. \label{eqn:lower}
\end{align}
The infimum is taken over all possible estimator $\hat{f}_n$ and the supremum is taken over all possible clean distribution $P$ that satisfy Assumption \ref{assm:margin} and \ref{assm:anchor}, as well as noisy distribution $\tilde{P}$, where $P$ and $\tilde{P}$ has the same marginal distribution over $\bX$.
\end{theorem}
Theorem \ref{thm:lower} suggests that even when provided with a sufficiently large number of noisy samples, there is no estimator capable of learning the optimal function $f^*$. 
This occurs because the noise rate varies for each instance, making the presence of anchor points insufficient to accurately estimate these noise rates.
Interestingly, when combined with Theorem \ref{thm:upper}, Theorem \ref{thm:lower} suggests that empirical risk minimization achieves the optimal rate of excess risk, which is proportional to the noise level $\rho$. 
Also, this implies that no other sophisticated methods such as loss correction or label correction are necessary without any additional assumptions.

\subsection{Proof of Theorem \ref{thm:lower}}
The two representative techniques to lower bound the minimax risk are Le Cam's method and Fano's method \citep{wainwright2019high}. Both approaches reduce the estimation problem to the testing problem and subsequently derive a lower bound for the probability of error within that testing context. Here, we use a variant of Fano's method \citep{birg2001new} to derive a minimax lower bound.
\begin{lemma}[\citet{birg2001new}]\label{lem:birg}
Let $\calp$ be a finite set of probability distributions in a measurable space $(\calz,\cala)$ with $|\calp| = m$. Suppose $Z$ is a random variable that follows a given distribution $P \in \calp$ and $D(P \|Q)$ is the Kullback-Leibler divergence between two distributions $P$ and $Q$. Then, we have
\begin{align*}
    \underset{\hat{\phi}}{\inf} \ \underset{P \in \calp}{\sup}  \left\{ \bbe_{Z \sim P} \left[ \mathds{1}\left\{\hat{\phi}(Z) \neq P \right\} \right] \right\} \geq \min \left\{ 0.36, 1- \underset{P \in\calp}{\inf} \left\{ \sum_{Q \in \calp} \frac{D(P\|Q))}{m\log (m)} \right\} \right\},
\end{align*}
where the infimum is taken over all possible estimators $\hat{\phi}: \calz \rightarrow \calp$.
\end{lemma}
\par

 Throughout this section, we use $+$ and $-$ to simplify the labels $+1$ and $-1$, respectively. The proof of Theorem \ref{thm:lower} is adapted to some extent from the proof by \citet{reeve2019classification}, which also uses Fano's method to derive a minimax lower bound for the CCN model. A key difference is that our study of the more general ILN model allows us to construct two distinct clean distributions, $P^{-}$ and $P^{+}$, while ensuring that their noisy counterparts, $\tilde{P}^{-}$ and $\tilde{P}^{+}$, are indistinguishable.
In this setting, distinguishing $P^{-}$ from $P^{+}$ becomes possible with a large enough sample size, while differentiating $\tilde{P}^{-}$ from $\tilde{P}^{+}$ remains impossible. 
As a result, any estimator that relies on noisy samples to determine the originating distribution of the samples would be left with no option but to make a random guess.
\par 
We choose $\calx:= \{a,b,c\}$ to be a finite set, and the marginal distribution over $\bX$ to be the same for $P^{-}$ and $P^{+}$. Consequently, we only need to concentrate on $\eta^i(\bx)$, which represents the conditional distribution of $Y=1$ given $\bX$ for the distribution $P^i$. We denote the noisy counterpart of $\eta^i(\bx)$ as $\tilde{\eta}^i(\bx)$ for $i \in \{-,+\}$. We set $\eta^i(a) = 1$ and $\eta^i(c) = 0$ for all $i \in \{-,+\}$, ensuring that both distributions adhere to the anchor point assumption. We set $\eta^{-}(b) < \frac{1}{2}$ and $\eta^{+}(b) > \frac{1}{2}$, as we want $\bx = b$ to be the point that distinguishes $P^{-}$ and $P^{+}$. This is shown in Figure \ref{fig:clean dist}.
\begin{figure}[t]
         \centering
     \begin{subfigure}{0.45\textwidth}
         \includegraphics[width=\textwidth]{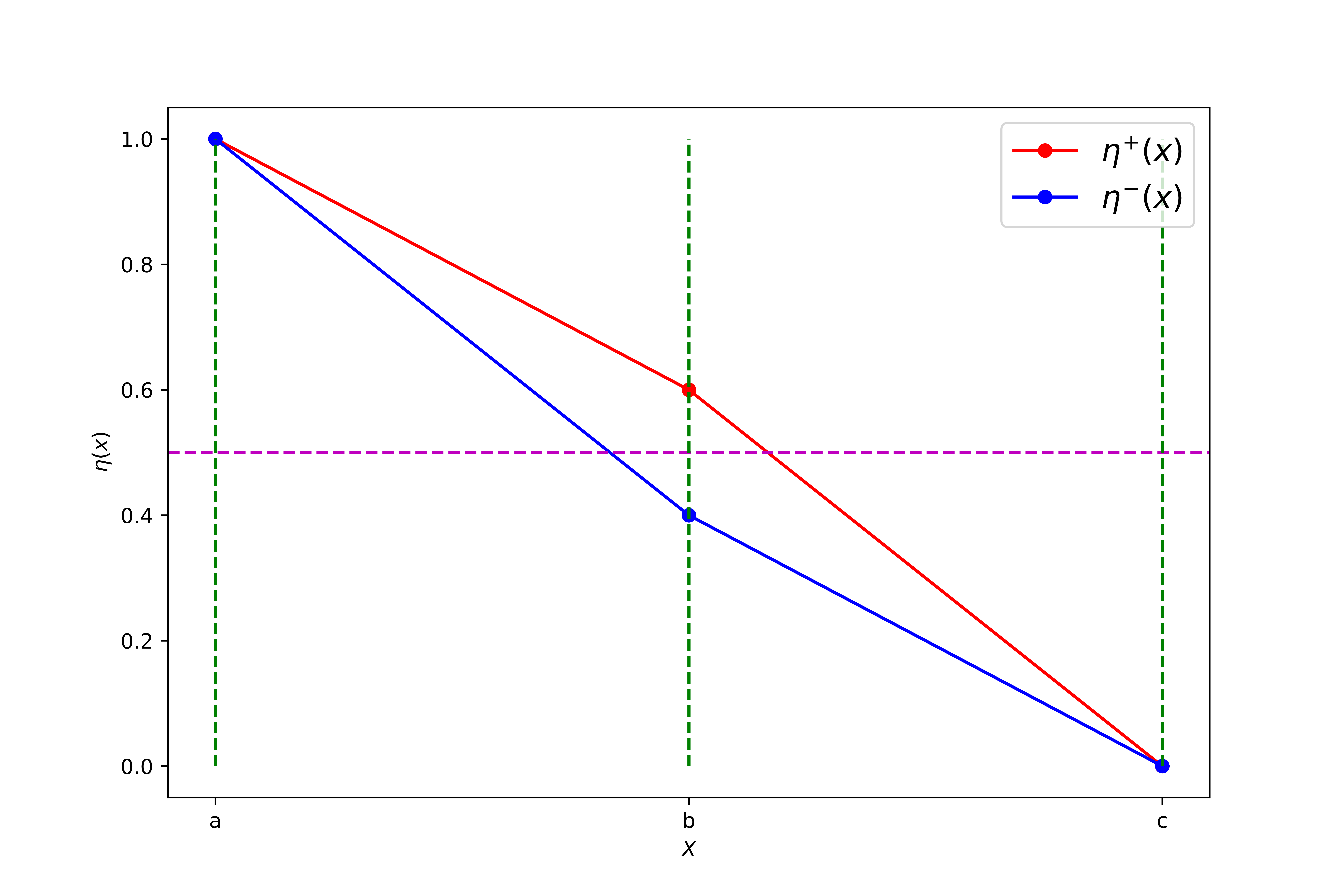}
         \caption{$P^-, P^+$}
         \label{fig:clean dist}
     \end{subfigure}
     \hfill
     \begin{subfigure}{0.45\textwidth}
         \includegraphics[width=\textwidth]{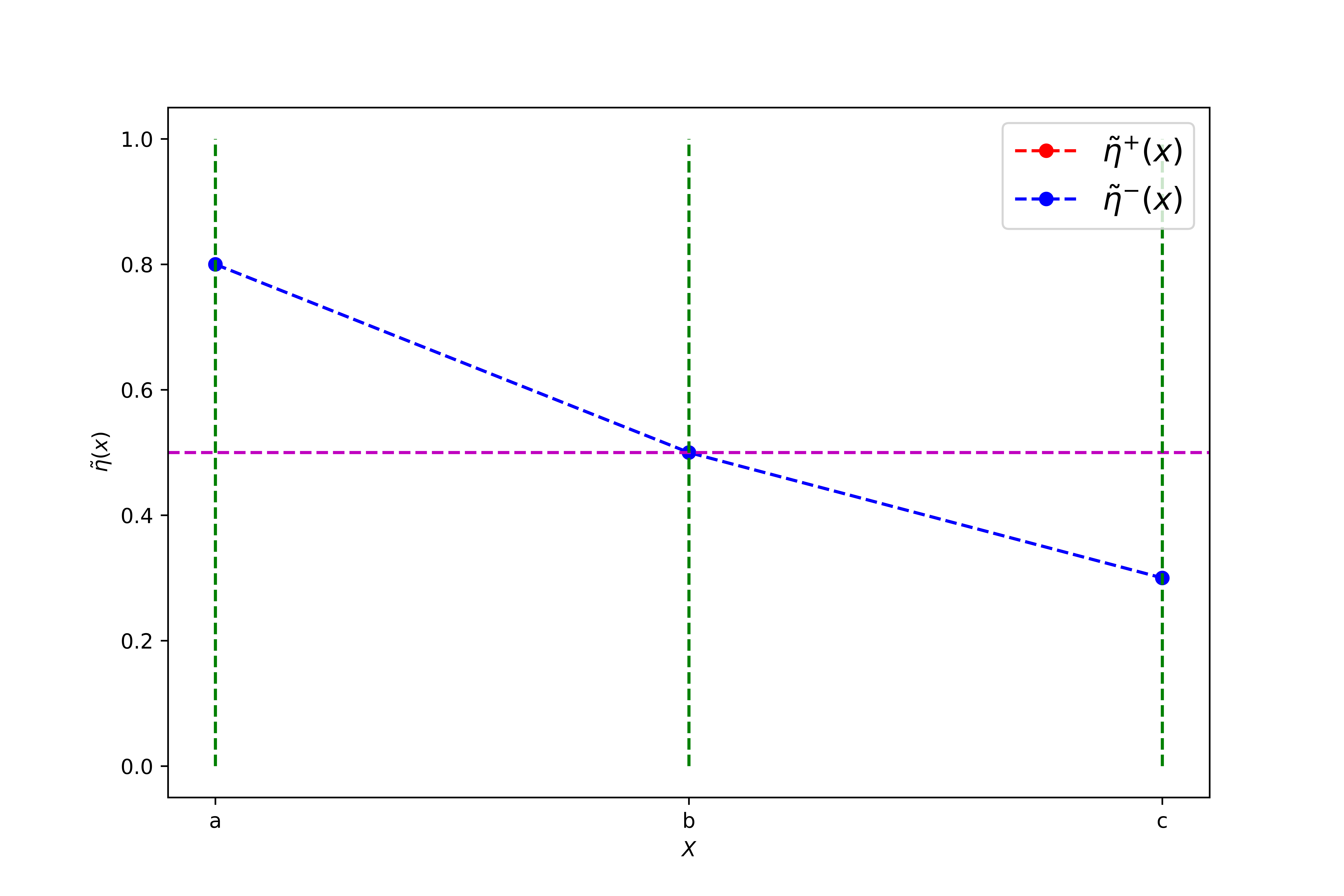}
         \caption{$\tilde{P}^-, \tilde{P}^+$}
         \label{fig:noisy dist}
     \end{subfigure}
    \caption{An illustration of the conditional probability of $P^-$ and $P^+$ and their noisy counterpart.}
    \label{fig:reeve fig2}
\end{figure}
We use $\rho^{i}_y(\bx)$ for $y \in \{-,+\}$ to denote the noise function of $\tilde{P}^i$ for $i \in \{-,+\}$.
We aim to configure the noise functions in such a way that the conditional distributions $\tilde{\eta}^{-}(\bx)$ and $\tilde{\eta}^{+}(\bx)$ become indistinguishable. By Lemma \ref{lem:eta}, we have
\begin{align}
    \tilde{\eta}^i(\bx) = (1-\rhop^i(\bx) - \rhom^i(\bx))\eta^i(\bx) + \rhom^i(\bx), \quad \forall i \in \{-,+\}. \label{eqn:eta}
\end{align}
Since $\eta^{+}(a) = \eta^{-}(a) =1$ and $\eta^{+}(c) = \eta^{-}(c) =0$, we choose $\rho^+_y(\bx) = \rho^-_y(\bx)$ for all $\bx = a,c$ and $y \in \{+1,-1\}$ resulting in $\tilde{\eta}^+(\bx) = \tilde{\eta}^-(\bx)$ for all $\bx = a,c$. 
We still need to determine the noise functions for $\bx = b$. To do so, we select $\eta^+(b) = \frac{1}{2- \rho}$ and $\eta^-(b) = \frac{1-\rho}{2-\rho}$.
For the noise function, we set $\rho^+_-(b) = 0$, $\rho^+_+(b) = \frac{\rho}{2}$, $\rho^-_-(b)= \frac{\rho}{2}$ and $\rho^-_+(b)= 0$. Then, we have $\tilde{\eta}^i(b) = \frac{1}{2}$ for all $i \in \{-,+\}$, which is illustrated in Figure \ref{fig:noisy dist}. Consequently, we have
\begin{align}
    |2\eta^i(b) -1| = \frac{\rho}{2-\rho} \geq \frac{\rho}{2}, \quad \forall i \in \{-,+\}. \label{eqn:example lower bound}
\end{align}
Equation \eqref{eqn:example lower bound} implies that the distribution $P^+$ and $P^-$ satisfy the margin assumption for $\rho >0$. 
We introduce the key lemma for the proof of our theorem. 
\begin{lemma}\label{lem:lower bound}
Let $\tilde{P}_n$ ($n$ product distribution of $\tilde{P}$) measurable classifier $\hat{f}_n$ be given. Then, we have
\begin{align*}
    \sum_{i\in \{-,+\}} \left( \bbe_{\tilde{S}_n \sim \tilde{P}^i_n}\left[ \calr^{0/1}(\hat{f}_n) \right] - \calr^{0/1}\left(f^{*,i}\right) \right) \geq \frac{\rho}{8},
\end{align*}
where $f^{*,i}:= \eta^i(\bx)-\frac{1}{2}$ for $i \in \{-,+\}$.
\end{lemma}

The proof of Theorem \ref{thm:lower} is straightforward from Lemma \ref{lem:lower bound}.
\begin{proof}[Proof of Theorem \ref{thm:lower}]
        \begin{align*}
        \underset{P, \tilde{P}}{\sup} \left\{ \mathbb{E}_{S_n \sim \tilde{P}_n} \left[ \calr(\hat{f}_n) \right] - \calr(f^*) \right\} \geq \frac{1}{2} \sum_{i \in \{-,+\}} \left( \bbe_{S_n \sim \tilde{P}^i_n}\left[ \calr(\hat{f}_n) \right] - \calr(f^{*,i}) \right) \geq \frac{\rho}{16},
    \end{align*}
    where the last inequality comes from Lemma \ref{lem:lower bound}. Since, the above inequality holds for given $\hat{f}_n$, Theorem \ref{thm:lower} is valid. 
\end{proof}

\section{Conclusion}\label{sec:conclusion}

This paper provides a theoretical analysis of \textit{instance and label dependent} noise model. We derived an upper bound for the excess risk of the empirical risk minimizer when using noisy samples. Furthermore, we showed that the minimax risk of the 0-1 loss is lower bounded by a constant that is proportional to the noise level. These results imply that the empirical risk minimizer achieves the optimal excess risk bound in terms of the noise level, despite its simplicity compared with other noise-tolerant methods. Furthermore, the lower bound result prompts us to consider what additional assumptions are required to make learning with ILN possible and some justification for using clean samples, as learning solely with noisy samples is not possible.
These theoretical findings validate numerous works on the ILN model, which assumed stronger assumptions and used human experts to clean noisy samples.

\bibliographystyle{abbrvnat}
\bibliography{ref.bib}

\newpage
\begin{appendix}

\section{Proofs}
\subsection{Proof for Lemma \ref{lem:eta}}
\begin{proof}
For given $x \in \calx$, we have
\begin{alignat*}{3}
    \bbp(\tilde{y}=1|\bX=\bx) &= &&\bbp(\tilde{y}=1,y=+1|\bX=\bx) + \bbp(\tilde{y}=1,y=-1|\bX=\bx) \\
    &= &&\bbp(\tilde{y}=1|y=+1,\bX=\bx)\cdot \bbp(y=+1|\bX=\bx) \\
    & &&+ \bbp(\tilde{y}=1|y=-1,\bX=\bx)\cdot\bbp(y=-1|\bX=\bx) \\
    &= &&(1- \rhop(\bx))\bbp(y=+1|\bX=\bx) + \rhom(\bx)\bbp(y=-1|\bX=\bx)\\
    &= &&(1- \rhop(\bx))\bbp(y=+1|\bX=\bx) + \rhom(\bx)(1-\bbp(y=+1|\bX=\bx))\\
    &= &&(1-\rhop(\bx) -\rhom(\bx))\bbp(y=1|\bX=\bx) + \rhom(\bx).\\
    &= &&(1-\rhop(\bx))\bbp(y=1|\bX=\bx) + \rhom(\bx)(1-\bbp(y=1|\bX=\bx))\\
    &= &&(1-\rhop(\bx))\eta(\bx) + \rhom(\bx)(1-\eta(\bx))
\end{alignat*}
The third equation follows from the definition of $\rhop(\bx)$ and $\rhom(\bx)$.
Similary, we have the following equation:
\begin{align}
    \bbp(\tilde{y}=-1|\bX=\bx) = (1-\rhop(\bx) -\rhom(\bx))\bbp(y=-1|\bX=\bx) + \rhop(\bx). \label{eqn:rhom}
\end{align}
\end{proof}
\subsection{Proof for Lemma \ref{lem:ub of clean-noisy}}
\begin{proof}
    Suppose $f\in \calf$ be given. We have
    \begin{alignat*}{3}
         \calr(f) - \tcalr(f) &= &&\int \ell(f(\bX),Y) dP - \int \ell(f(\bX),\tilde{Y}) d\tilde{P}\\
         &= &&\int \bbp (y=1|\bX) \cdot \ell(f(\bX),+1) + \bbp (y=-1|\bX) \cdot \ell(f(\bX),-1) dP_\bX\\
         & &&- \int \bbp (\tilde{y}=1|\bX) \cdot \ell(f(\bX),+1) + \bbp (\tilde{y}=-1|\bX) \cdot \ell(f(\bX),-1) dP_\bX.
    \end{alignat*}
    The first equation follows from the definition of the clean risk $\calr(f)$ and the noisy risk $\tcalr(f)$. The second equation is derived from $\bbe [\ell(f(\bX),Y)] = \bbe [\bbe[\ell(f(\bX),Y)|\bX]]$.

    Given that both terms in the preceding equation are integrated over $\bX$, we can combine the two terms accordingly:
    \begin{alignat*}{3}
        \calr(f) - \tcalr(f) &= &&\int \Big\{ \bbp (y=1|\bX) \cdot \ell(f(\bX),+1) -  \bbp (y=-1|\bX) \cdot \ell(f(\bX),-1) \\
         & &&- \bbp (\tilde{y}=1|\bX) \cdot \ell(f(\bX),+1) - \bbp (\tilde{y}=-1|\bX) \cdot \ell(f(\bX),-1) \Big\} dP_\bX. 
    \end{alignat*}
    Using \eqref{eqn:rhop} and \eqref{eqn:rhom} and rearranging the terms, we have
    \begin{alignat*}{3}
        \calr(f) - \tcalr(f) &= && \int \big\{(\rhop(\bX) + \rhom(\bX))\eta(\bX)\ell(f(\bX),+1) - \rhop(\bX) \ell(f(\bX),-1)  \big\}\\
        & &&+ \big\{(\rhop(\bX) + \rhom(\bX))\left(1-\eta\left(\bX\right)\right)\ell(f(\bX),-1)  - \rhom(\bX) \ell(f(\bX),+1) \big\} dP_\bX.
    \end{alignat*}
    Let us focus on the first bracket term within the integral. For given $\bx \in \calx$, we have 
    \begin{align*}
        &(\rhop(\bx) + \rhom(\bx))\eta(\bx)\ell(f(\bx),+1) - \rhop(\bX) \ell(f(\bX),-1) \\
        &= (\rhop(\bx) + \rhom(\bx))\left(\eta(\bx)\ell(f(\bx),+1)- \frac{\rhop(\bx)}{\rhop(\bx)+\rhom(\bx)} \ell(f(\bx),-1)\right)
    \end{align*}
    We can split the second bracket term as follows
    \begin{align*}
        &\eta(\bx)\ell(f(\bx),+1)- \frac{\rhop(\bx)}{\rhop(\bx)+\rhom(\bx)} \ell(f(\bx),-1)= 
         \underbrace{\left(\eta(\bx)\ell(f(\bx),+1)- \frac{\rhop(\bx)}{\rhop(\bx)+\rhom(\bx)} \ell(f(\bx),+1)\right)}_{=(A)} \\
        &+\left(\frac{\rhop(\bx)}{\rhop(\bx)+\rhom(\bx)} \ell(f(\bx),+1) -\frac{\rhop(\bx)}{\rhop(\bx)+\rhom(\bx)} \ell(f(\bx),-1)\right)\\
        &\leq (A) + \frac{C\rhop(\bx)}{\rhop(\bx)+\rhom(\bx)}. 
    \end{align*}
    The last inequality comes from Assumption \ref{assm:constant C}. Similarly, the second bracket term within the integral can be upper bounded by
    \begin{align*}
        &(1-\eta(\bx))\ell(f(\bx),-1)  - \frac{\rhom(\bx)}{\rhop(\bx)+\rhom(\bx)} \ell(f(\bx),+1)\\
        &\leq \underbrace{\left((1-\eta(\bx))\ell(f(\bx),-1)- \frac{\rhom(\bx)}{\rhop(\bx)+\rhom(\bx)} \ell(f(\bx),-1)\right)}_{=(B)} + \frac{C\rhom(\bx)}{\rhop(\bx)+\rhom(\bx)}.
    \end{align*}
    Combining the above result, we have
    \begin{align*}
        \calr(f) - \tcalr(f) \leq \int (\rhop(\bx)+\rhom(\bx))\left( (A) + (B) + C \right) dP_\bx.
    \end{align*}
    Since $C$ is a constant, it remains to bound $(A) + (B)$.
    \begin{align}
        &(A) + (B) \notag \\
        &=\left( \eta(\bx) - \frac{\rhop(\bx)}{\rhop(\bx)+\rhom(\bx)} \right) \ell(f(\bx),+1) + \left(1- \eta(\bx) - \frac{\rhom(\bx)}{\rhop(\bx)+\rhom(\bx)} \right)\ell(f(\bx),-1)\notag \\
        &=\left( \eta(\bx) - \frac{\rhop(\bx)}{\rhop(\bx)+\rhom(\bx)} \right) \ell(f(\bx),+1) - \left(\eta(\bx) - \frac{\rhop(\bx)}{\rhop(\bx)+\rhom(\bx)} \right)\ell(f(\bx),-1) \notag\\
        &=\left( \eta(\bx) - \frac{\rhop(\bx)}{\rhop(\bx)+\rhom(\bx)} \right) \left( \ell(f(\bx),+1) - \ell(f(\bx),-1)  \right) \notag\\
        &\leq \left| \eta(\bx) - \frac{\rhop(\bx)}{\rhop(\bx)+\rhom(\bx)} \right| C \label{eqn:first}
    \end{align}
    The last inequality comes from Assumption \ref{assm:constant C}. Instead of binding by common factor $\eta(\bx) - \frac{\rhop(\bx)}{\rhop(\bx)+\rhom(\bx)}$ in the third equation, if we use $\eta(\bx) - \frac{\rhop(\bx)}{\rhop(\bx)+\rhom(\bx)} = \eta(\bx)-1 + \frac{\rhom(\bx)}{\rhop(\bx)+\rhom(\bx)}$, then we have
    \begin{align}
        (A) +(B)&\leq \left(\eta(\bx)-1 + \frac{\rhom(\bx)}{\rhop(\bx)+\rhom(\bx)}\right) \left( \ell(f(\bx),+1) - \ell(f(\bx),-1) \right) \notag \\
        &= \left(1- \eta(\bx) - \frac{\rhom(\bx)}{\rhop(\bx)+\rhom(\bx)}\right)
        \left( \ell(f(\bx),-1) - \ell(f(\bx),+1)\right) \notag \\ 
        &\leq \left|1- \eta(\bx) - \frac{\rhom(\bx)}{\rhop(\bx)+\rhom(\bx)} \right| C \label{eqn:second}
    \end{align}
    Using \ref{eqn:first} and \eqref{eqn:second}, we get
    \begin{align*}
        (A) +(B) \leq \frac{C \min\{\rhop(\bx),\rhom(\bx)\}}{\rhop(\bx) +\rhom(\bx)} \leq \frac{C}{2},
    \end{align*}
    where the last inequality comes from $\rhop(\bx) + \rhom(\bx) < 1$.
    As a result, we have 
    \begin{align*}
         \calr(f) - \tcalr(f) \leq \frac{3C}{2} \int \rhop(\bX)+\rhom(\bX) dP_\bX.
    \end{align*}
    Since $f$ is given, we have
    \begin{align*}
         \underset{f \in \calf}{\sup} \big\{ \calr(f) - \tcalr(f) \big\} \leq \frac{3C}{2} \int \rhop(\bX)+\rhom(\bX) dP_\bX \leq \frac{3C\rho}{2}.
    \end{align*}
\end{proof}

\subsection{Proof for Theorem \ref{thm:upper}}
\begin{proof}
First of all, we have
    \begin{align*}
    \calr(\nes) - \calr (f_n) &= \calr(\nes) - \tilde{\calr}(\nes) +\tilde{\calr}(\nes) -  \calr (f_n)\\
    &\leq \underset{f \in \calf}{\sup} \ \left\{ \calr(f) - \tilde{\calr}(f) \right\} +\tilde{\calr}(\nes) -  \calr (f_n) \\
    &\leq \frac{3C\rho}{2} +\tilde{\calr}(\nes) -  \calr (f_n).
\end{align*}
The last inequality comes from Lemma \ref{lem:ub of clean-noisy}. We decompose the second term in the above inequality to
\begin{align*}
    \tilde{\calr}(\nes) -  \calr (f_n) = \underbrace{\tilde{\calr}(\nes) - \tilde{\calr}_n (\nes)}_{(a)} + \tilde{\calr}_n (\nes) - \calr_n (f_n) + \underbrace{\calr_n (f_n) -  \calr (f_n)}_{(b)}. 
\end{align*}
The following holds with probability at least $1-2\delta$ by Assumption \ref{assm:ERM}:
\begin{align*}
    (a)+(b) \leq 2G_\delta(n).
\end{align*}
Thus, with probability at least $1-\delta$, we get
\begin{align*}
    \tilde{\calr}(\nes) -  \calr (f_n) \leq  \frac{3C\rho}{2} + 2G_\delta(n) + \tilde{\calr}_n (\nes) - \calr_n (f_n).
\end{align*}
It remains to bound the middle term $ \tilde{\calr}_n (\nes) - \calr_n (f_n)$. We have
\begin{align*}
    \tilde{\calr}_n (\nes) - \calr_n (f_n) &= \underbrace{\tilde{\calr}_n (\nes) - \tilde{\calr}_n (f_n)}_{\leq 0} +\tilde{\calr}_n (f_n)- \calr_n (f_n)\\
    &\leq \underset{f \in \calf}{\sup} \ \tilde{\calr}_n (f) - \calr_n (f)\\
    &\leq \underset{f \in \calf}{\sup} \tilde{\calr}(f) - \calr(f)\\
    &\leq  \frac{3C\rho}{2}.
\end{align*}
The first inequality comes from the optimality of $\nes$ and the second equality comes from $\max_i\{a_i + b_i\} \leq \max_i \{a_i\} + \max_i \{b_i\}$. Therefore, we obtain
\begin{align*}
    \calr(\nes) - \calr(f_n) \leq 3C\rho + 2G_\delta(n),
\end{align*}
with probability at least $1-2\delta$.
\end{proof}

\subsection{Proof for Lemma \ref{lem:lower bound}}
\begin{proof}
Suppose $\tilde{S}_n$ be the given $n$ noisy samples, but we do not know whether $\tilde{S}_n$ came from $\tilde{P}^-_n$ or $\tilde{P}^+_n$.
    We want to construct an estimator $\hat{\phi}: \left( \calx \times \caly\right)^n \rightarrow \{\tilde{P}^-_n,\tilde{P}^+_n\}$, which estimates the originating  distribution of the samples, using the function $\hat{f}_n$ that is trained using $\tilde{S}_n$. We choose $\hat{\phi}\left(\tilde{S}_n \right) = \tilde{P}^i_n$, if $\sgn(\hat{f}_n(b)) = i$. Since, $\eta^-(b) < \frac{1}{2}$ and $\eta^+(b) >\frac{1}{2}$, Bayes optimal classifier, $\sgn(f^{*,i}(\bx))$ for each distribution $P^i$ satisfies $\sgn(f^{*,i}(b)) = i$ for $i \in \{-,+\}$.
    Then, we have 
    \begin{align*}
        \sum_{i \in \{-,+\}} \bbe_{\tilde{S}_n\sim \tilde{P}^i_n} \left[ \mathds{1}\left\{\sgn( \hat{f}_n(b)) \neq \sgn(f^{*,i}(b)) \right\} \right] = \sum_{i \in \{-,+\}} \bbe_{\tilde{S}_n\sim \tilde{P}^i_n} \left[\mathds{1} \left\{ \hat{\phi}\left(\tilde{S}_n \right) \neq \tilde{P}_n^i \right\} \right],
    \end{align*}
    by definition of $\hat{\phi}$. Since $\tilde{P}^-_n$ and $\tilde{P}^+_n$ are the same, we have $D(\tilde{P}^-_n\| \tilde{P}^+_n) = 0$. By Lemma \ref{lem:birg}, we have
    \begin{align}
        \sum_{i \in \{-,+\}} \bbe_{\tilde{S}_n\sim \tilde{P}^i_n} \left[\mathds{1} \left\{ \hat{\phi}\left(\tilde{S}_n \right) \neq \tilde{P}_n^i \right\} \right] \geq \frac{1}{3}. \label{eqn:temp1}
    \end{align}
    Also, the following holds for any $\tilde{P}_n$ measurable $f$ and $i \in \{-1,+1\}$.
    \begin{align*}
        \calr^{0/1}\left( f \right) - \calr^{0/1}\left( f^{*,i} \right) &= \sum_{\bx \in \calx} \bbp(\bX=\bx) \left| 2\eta^i(\bx) -1 \right|\mathds{1}\left\{ \sgn(f(\bx)) \neq \sgn(f^{*,i}(\bx))  \right\} \\
        &\geq \bbp(\bX=b)|2\eta^i(b) -1 |\mathds{1}\left\{ \sgn(f(b)) \neq \sgn(f^{*,i}(b))  \right\} \\
        &\geq \bbp(\bX=b)\frac{\rho}{2} \mathds{1}\left\{\sgn( f(b)) \neq \sgn(f^{*,i}(b))  \right\},
    \end{align*}
    where the last inequality comes from \eqref{eqn:example lower bound}. Inputting $\hat{f}_n$ for $f$ and taking expectation over $\tilde{S}_n \sim \tilde{P}^i_n$ for $i \in \{-,+\}$, we get
    \begin{align*}
        \bbe_{\tilde{S}_n \sim \tilde{P}^i_n}\left[ \calr^{0/1}(\hat{f}_n) \right] - \calr^{0/1}\left(f^{*,i}\right) \geq\bbp(\bX=b)\cdot\frac{\rho}{2} \cdot \bbe_{\tilde{S}_n \sim \tilde{P}^i_n}\left[\mathds{1}\left\{ \sgn(f(b)) \neq \sgn(f^{*,i}(b))  \right\} \right].
    \end{align*}
    Summing the above inequality for $i \in \{-,+\}$ and using \eqref{eqn:temp1},  we get
\begin{align*}
    \sum_{i\in \{-,+\}} \left( \bbe_{\tilde{S}_n \sim \tilde{P}^i_n}\left[ \calr^{0/1}(\hat{f}_n) \right] - \calr^{0/1}\left(f^{*,i}\right) \right)\geq \frac{\bbp(\bX=b)\cdot \rho}{6}.
\end{align*}
As we can choose $\bbp(\bX=b) =\frac{3}{4}$, we get the desired result. 
\end{proof}
\end{appendix}

\end{document}